\def\DIBSAUTHORPREVIEW{1}
\documentclass{article}

\usepackage{microtype}
\usepackage{graphicx}
\graphicspath{{figures/}}
\usepackage{subcaption}
\usepackage{booktabs}
\usepackage{hyperref}

\ifdefined\DIBSAUTHORPREVIEW
\usepackage[preprint]{icml2026}
\else
\usepackage{icml2026}
\fi
\usepackage{amsmath,amssymb,amsfonts,amsthm}
\usepackage{array,multirow}
\usepackage[T1]{fontenc}
\usepackage{textcomp}
\usepackage{verbatim}
\usepackage{nicefrac}
\usepackage{xcolor}

\newcommand{\ablworse}[1]{\textcolor{red!70!black}{#1}}
\newcommand{\ablbetter}[1]{\textcolor{green!50!black}{#1}}
\newtheorem{theorem}{Theorem}

\icmltitlerunning{DiBS: Diffusion-Informed Branch Selection}
\ifdefined\DIBSAUTHORPREVIEW
\newcommand{\dibsauthor}[2]{\icmlauthor{#1}{#2}}
\else
\newcommand{\dibsauthor}[2]{\icmlauthor{#1}{affiliation}}
\fi

\begin{document}
\twocolumn[
\icmltitle{DiBS: Diffusion-Informed Branch Selection}
\begin{icmlauthorlist}
\dibsauthor{Bo Liu}{cuhksz}
\dibsauthor{Yuan Xie}{cuhksz}
\dibsauthor{Yuan Gao}{cuhksz}
\dibsauthor{Xiaolong Luo}{polyu}
\dibsauthor{Jingqi Ye}{ustc}
\dibsauthor{Peng Ye}{mmlab,shanghaiai}
\dibsauthor{Tao Chen}{fudan}
\dibsauthor{Fujun Han}{cuhksz}
\end{icmlauthorlist}
\ifdefined\DIBSAUTHORPREVIEW
\icmlaffiliation{cuhksz}{The Chinese University of Hong Kong, Shenzhen, China}
\icmlaffiliation{polyu}{The Hong Kong Polytechnic University}
\icmlaffiliation{ustc}{University of Science and Technology of China}
\icmlaffiliation{mmlab}{MM Lab, The Chinese University of Hong Kong}
\icmlaffiliation{shanghaiai}{Shanghai AI Laboratory}
\icmlaffiliation{fudan}{Fudan University}
\icmlcorrespondingauthor{Fujun Han}{hanfujun@cuhk.edu.cn}
\else
\icmlaffiliation{affiliation}{Affiliation details to be confirmed for the final version}
\icmlcorrespondingauthor{Bo Liu}{contact-to-confirm@example.invalid}
\fi
\icmlkeywords{Diffusion Models, Constraint Programming, Sudoku, Search}
\vskip 0.3in
]
\printAffiliationsAndNotice{}

\begin{abstract}
Sudoku is a representative constraint satisfaction problem that requires global structural reasoning under strict discrete constraints. Learning-based solvers capture global patterns in Sudoku. Complete symbolic solvers provide exact answers, and costly wrong branches can lead to long-tail runtimes. We introduce DiBS, which uses diffusion-informed global guidance to order values inside a complete constraint solver. The model ranks legal values for the variable selected by the solver. The solver propagates constraints, validates solutions, and backtracks exhaustively as needed. DiBS aligns the guidance with the solver's actual objective by learning counterfactual search costs from its propagated domains and action-conditioned search footprints, and a compact one-pass model serves batched online queries. We prove a prediction-error bound for the resulting action choice and identify the exact regime in which the wrong-subtree saving identity holds. The action cost includes adaptive search after the first choice, and the runtime evaluation includes the cost of guidance.
\end{abstract}

\section{Introduction}
\label{sec:introduction}

Sudoku is a constraint satisfaction problem (CSP) where classical solvers rely on depth-first backtracking with constraint propagation (CP) and heuristics such as minimum remaining values (MRV) \cite{mackworth1977consistency,haralick1980increasing,freuder1982sufficient,regin1994filtering,simonis2005sudoku}. These solvers are complete: they find a solution if one exists. On hard instances, however, performance depends on both propagation and the structure of the search tree. A few early wrong branching decisions can lead to large failing subtrees whose contradictions appear only deep in search, producing long-tail runtimes \cite{dechter1992backjumping,harvey1995limited,walsh1997depth}.

\begin{figure}[t]
\centering
\includegraphics[width=0.88\linewidth]{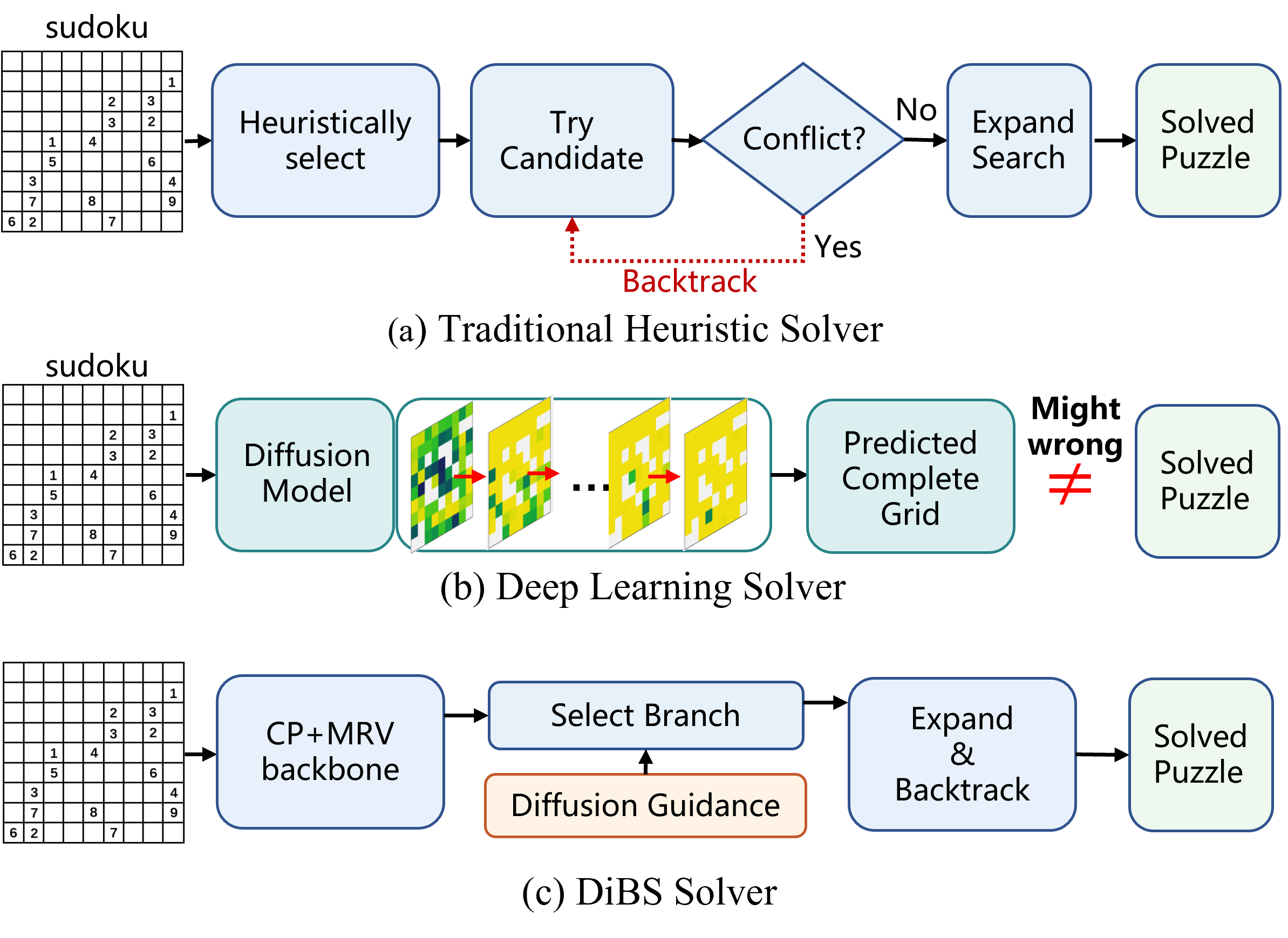}
\caption{Comparison of three solver paradigms. DiBS adds a third paradigm: a learned model orders legal values while a complete solver retains exact search.}
\label{fig:solver_methods}
\vspace{-20pt}
\end{figure}

Recently, learned models such as recurrent relational networks and differentiable solvers have been explored to capture Sudoku's global structure, and typically predict a completed grid \cite{palm2018recurrent,wang2019satnet,meng2021sudoku}. We study how a generative model can guide the ordering of values during complete search. Discrete diffusion models \cite{ho2020denoising,austin2021structured,campbell2022ctdd,weilbach2023graphically} offer a promising source of global compatibility information beyond what local propagation alone exposes. Existing diffusion-based combinatorial methods typically generate complete solutions or plans \cite{sun2023difusco,janner2022diffuser,ye2025beyond}.

We use learned global information to order legal values at the root of complete search. If a model helps the solver try a good value before an expensive failing value, it avoids exploring that failing subtree before finding a solution. This effect is especially relevant on hard instances, where an incorrect value can remain locally consistent for many steps. Learned branching has improved search decisions in exact solvers \cite{khalil2016learning,gasse2019exact,scavuzzo2022treemdp}. The wrong-subtree argument explains the saving when subsequent search is unchanged by a failed branch. Because an adaptive solver updates its future decisions after a failure, DiBS measures the full search cost incurred by the solver after each first choice.

In this paper, we use diffusion-informed guidance to order candidate values within a complete CP solver. The solver propagates constraints and backtracks over the legal values. The model receives exact propagated domains and learns from the counterfactual cost and spatial search footprint of trying each candidate under Choco's continuation. Full solutions remain useful for pretraining, while solver-generated costs train the choice made online. At inference, one compact model pass scores all candidate values at the queried root, with puzzles batched to amortize GPU work. Figure~\ref{fig:solver_methods} illustrates the solver-guidance paradigm.

In total, our contributions can be summarized as follows:
\begin{itemize}
    \item \textbf{A novel paradigm:} we use diffusion-informed guidance to prioritize legal values for a solver-selected variable. The solver handles propagation, backtracking, and solution validation.
    \item \textbf{A targeted training strategy:} we train on counterfactual action costs and search footprints from the same propagated state, then distill candidate preferences into a student that scores all legal values in one pass.
    \item \textbf{A comprehensive improvement:} we reduce mean search events, backtracks, and end-to-end time across all seven Sudoku solvers, including guidance overhead, and increase the number of SAT formulas solved within budget across all six solvers.
\end{itemize}

\section{Related Work}

Value ordering has long been studied within complete constraint search. Classical heuristics use look-ahead compatibility \cite{frost1995lookahead}, propagation impact \cite{refalo2004impact}, or values' observed survival during search \cite{zhang2008learned}. Its effect depends on branching semantics: propagating a failed value in binary search can change later work even when all solutions are sought \cite{smith2005value}. On scheduling CSPs, harder cases were found to offer fewer consequential choices but higher penalties for individual mistakes \cite{luchterhand2025impact}. Learning-based work has addressed CSP variable ordering with a leaf-cost surrogate \cite{song2021variable}, branching rules across instance distributions \cite{balcan2018learning}, and global stochastic CSP search \cite{tonshoff2023one}. Within complete CP, SeaPearl supports learned branching \cite{chalumeau2021seapearl}, while Marty et al. learn value selection with deep Q-learning and propagation-based rewards \cite{marty2024value}. DiBS is closest to the latter, but supervises each root value with the complete solver's counterfactual node cost and spatial search footprint. This avoids using a local proxy as the final objective, a mismatch also identified for branch-and-cut \cite{cheng2026theoretical}.

Diffusion models have also been applied to constraint satisfaction: guided discrete diffusion generates Sudoku solutions without supervised solution labels \cite{jung2025guided}, and compositional diffusion generates assignments for continuous constraints in robot planning \cite{yang2023compositional}. These approaches use diffusion to construct candidate assignments. DiBS denoises an action's search footprint and uses the predicted cost to order legal values inside complete Choco search. This turns structured denoising into an action-cost signal that carries into complete search across solver backbones.

\section{Preliminaries}
\label{sec:preliminaries}

\subsection{CSPs and Complete Search}
A constraint satisfaction problem (CSP) consists of variables $V$, their finite domains $D=\{D_v\}_{v\in V}$, and constraints $\mathcal C$~\cite{apt2003principles,dechter2003constraint}. In Sudoku, the 81 cells are variables with domains in $\{1,\ldots,9\}$; clues fix singleton domains, and 27 all-different constraints enforce distinct digits in every row, column, and box~\cite{regin1994filtering,simonis2005sudoku}. A complete CP solver repeatedly propagates these constraints, chooses an unassigned variable, tries a legal value, and backtracks on failure. Choco uses the adaptive dom/wdeg rule for variable selection; trying a different legal value first changes search effort while preserving completeness.

Write a solver state as $s=(x,D,h)$, where $x$ is the partial grid, $D$ the propagated domains, and $h$ the internal heuristic state. If the solver selects $v(s)$, its available actions are $A(s)=\{(v(s),d):d\in D_{v(s)}(s)\}$. We measure the consequence of trying action $a$ first under the otherwise fixed complete continuation $\mathcal S$ by
\begin{equation}
Q_{\mathcal S}(s,a)
:=\operatorname{Nodes}(\mathcal S\mid s,\;a\text{ first}).
\label{eq:counterfactual_q}
\end{equation}
Candidates are compared from the same root state. A failed branch can update heuristic weights and shorten the subsequent search enough to make a failing first action cheaper overall.

\subsection{Discrete Diffusion for Sudoku Reasoning}
A complete Sudoku solution can be represented as 81 categorical tokens $x_0\in\{1,\ldots,9\}^{81}$. A masked discrete diffusion model corrupts some tokens and learns to reconstruct them from the remaining grid context \cite{austin2021structured,ye2025beyond}. With a corruption distribution $q_t(\tilde x\mid x_0)$, a representative denoising objective is
\begin{equation}
\mathcal L_{\rm sol}(\theta)
=\mathbb E_{x_0,t,\tilde x\sim q_t(\cdot\mid x_0)}
\sum_{v\in M(\tilde x)}-\log p_\theta(x_{0,v}\mid\tilde x,t),
\end{equation}
where $M(\tilde x)$ is the set of corrupted cells. Reconstructing missing digits teaches global relationships among distant cells and provides an initialization for learning action costs. Our solver-conditioned model uses masked denoising as its representation-learning mechanism and predicts action costs and search footprints directly; Section~\ref{sec:method} defines the targets.

\section{Solver-Conditioned DiBS}
\label{sec:method}

\begin{figure*}[t]
\centering
\includegraphics[width=0.98\textwidth]{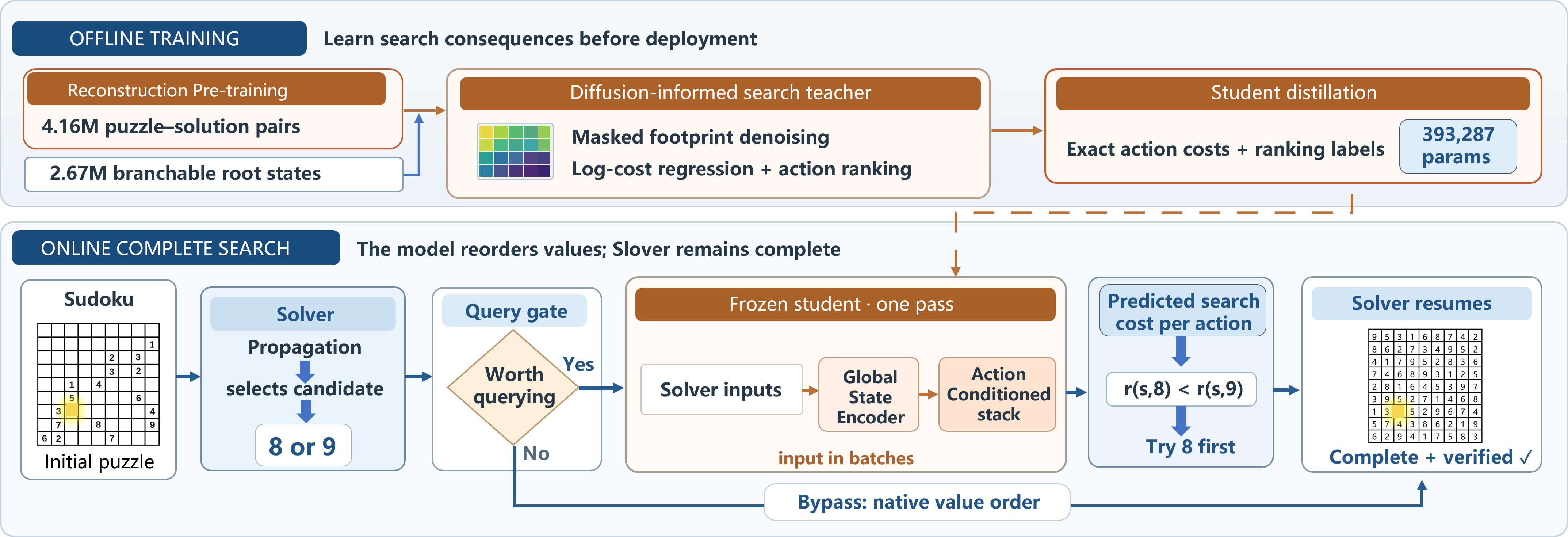}
\caption{DiBS training and deployment. Search-cost and footprint supervision train a compact student offline. Online, a gated, batched pass ranks legal values before the solver resumes complete search.}
\label{fig:dibs_pipeline}
\end{figure*}

\subsection{The Guidance Decision}
DiBS chooses which legal value the solver tries first for the selected variable at the root. Consider Choco's root state $s$ after propagation. Choco chooses $v(s)$ using dom/wdeg; DiBS scores each legal digit $d\in D_{v(s)}(s)$ and returns only a preferred first value. Choco retains the same constraints, variable heuristic, and remaining values. If the preferred value fails, its usual backtracking tries the others. Figure~\ref{fig:dibs_pipeline} depicts this exchange.

We measure the search consequence of each root action by counting the nodes Choco visits to completion when it tries that action first. This includes any failing subtrees and the subsequent search under Choco's updated heuristic weights. The model predicts $r_\theta(s,a)\approx\log(1+Q_{\mathcal S}(s,a))$ and chooses
\begin{equation}
\hat a=\arg\min_{a\in A(s)}r_\theta(s,a).
\label{eq:dibs_action}
\end{equation}
This target captures both the global structure of the puzzle and the cost of the solver's actual continuation. It also identifies the precise transferable interface: another solver chooses a different variable but still decides which legal value to favor at that variable.

\subsection{Learning Global Search Consequences}
The deployed model first encodes the current grid, an $81\times9$ legal-domain mask, and original clue indicators. Shared state blocks exchange information within each row, column, and box and across the whole grid. For every candidate $a=(v,d)$, an action-conditioned spatial stack then combines that shared representation with the selected cell and digit. Thus the model sees the entire puzzle when evaluating one local branch, and all candidates reuse the same state encoding. Its outputs are a log-cost prediction, an auxiliary solution-leading score, and a distribution over an 81-cell \emph{search footprint}.

We give this representation a progression of training signals. Solution pretraining reconstructs missing digits from 4.16 million checked, unique puzzle--solution pairs, teaching the encoder global Sudoku structure. Search labeling then covers 2.67 million unique, branchable training roots with exact Choco action costs. From the same propagated root, it tries each legal digit first in turn and records the full continuation cost $Q_{\mathcal S}(s,a)$. For each such continuation it also counts how often search selects each cell, including visits inside failing subtrees. These 81 counts are quantized into five bins ($0$, $1$, $2$--$3$, $4$--$7$, and $8+$). This footprint records \emph{where} the candidate makes the solver work, independent of chronological order. Root-solved and capped examples are tagged separately from complete action-cost labels.

The search teacher learns $\log(1+Q_{\mathcal S})$ with a robust regression loss and learns pairwise action order with a ranking loss weighted by the observed cost gap. A solution-leading loss and masked-footprint denoising loss provide complementary signals. Footprint bins are randomly hidden or occasionally replaced during training; fully masked examples match the deployed input. A larger teacher learns this structure offline, then a 393,287-parameter student learns both the exact cost/ranking labels and the teacher's soft candidate distribution. The student therefore delivers the search-consequence signal in a single, fully masked inference pass.

\subsection{Making Guidance Fast Enough to Use}
DiBS is designed to save more search time than guidance costs. We decide whether to query before expensive state capture: a cheap puzzle-level gate identifies worthwhile cases, and other cases continue natively. For queried cases, a compact representation of the current decision is sent to a resident model; states from many puzzles and all candidate values are scored in one GPU batch and one fully masked pass. A thin adapter translates the returned preference into the solver's own decision or hint interface, while the solver retains its state and complete search procedure. This combination amortizes device communication, avoids rebuilding the search state, and keeps the online output small. Implementation details are in Appendix~\ref{app:solver_adapters}.

\subsection{Why DiBS Is Useful}
Two candidates at the same propagated root can lead to different amounts of search. For Choco, each candidate is evaluated from the same initial heuristic state, and subsequent weight updates are included in the continuation. Let $q(a)=\log(1+Q_{\mathcal{S}}(s,a))$ and let $\hat a$ be the choice in Eq.~\eqref{eq:dibs_action}.

\begin{theorem}[Action-cost prediction bound]
\label{thm:cost_prediction}
If $\max_{a\in A(s)}|r_\theta(s,a)-q(a)|\leq\epsilon$, and $a^*=\arg\min_a Q_{\mathcal{S}}(s,a)$, then
\begin{equation}
1+Q_{\mathcal{S}}(s,\hat a)
\leq e^{2\epsilon}\bigl(1+Q_{\mathcal{S}}(s,a^*)\bigr).
\label{eq:cost_bound}
\end{equation}
If a student's log-cost predictions differ from a teacher's by at most $\eta$ on every candidate and the teacher satisfies the same premise with error $\epsilon_T$, replace $\epsilon$ by $\epsilon_T+\eta$.
\end{theorem}
\begin{proof}
At the fixed state $s$, write $r(a)=r_\theta(s,a)$. Since $\hat a$ minimizes $r$, we have $q(\hat a)\leq r(\hat a)+\epsilon\leq r(a^*)+\epsilon\leq q(a^*)+2\epsilon$. Exponentiation gives Eq.~\eqref{eq:cost_bound}, and the triangle inequality bounds the student's error by $\epsilon_T+\eta$.
\end{proof}

Prediction errors can overestimate the cheapest action and underestimate a competitor by $\epsilon$ each. If every competitor's log cost exceeds the optimum by more than $2\epsilon$, the predicted ordering still puts the cheapest action first. This margin connects the ranking target in Section~4.2 to the decision made online.

We isolate the regime that makes the wrong-subtree mechanism exact: exactly one of two child branches has a solution, exhausting the failing branch costs $W$, and exploring that failure leaves the subsequent solution search unchanged. Let $G$ be the fixed cost of the good-child continuation.

\begin{theorem}[Restricted binary-subtree identity]
\label{thm:binary_cost}
If policy $\pi$ selects the solution-leading child first with probability $p_s(\pi)$, its expected search cost, excluding model overhead, is
\begin{equation}
\mathbb{E}[C_\pi(s)]=G+\bigl(1-p_s(\pi)\bigr)W.
\label{eq:binary_decomp}
\end{equation}
For two such policies, $\mathbb{E}[C_{\pi_a}(s)]-\mathbb{E}[C_{\pi_b}(s)]=\bigl(p_s(\pi_b)-p_s(\pi_a)\bigr)W$. A model call adds its own measured cost.
\end{theorem}

The two realized costs are $G$ and $G+W$. Writing $C_\pi(s)=G+MW$ with $M$ indicating a failing first choice gives Eq.~\eqref{eq:binary_decomp} by expectation. On a fixed tree with a unique solution, each failing sibling visited before the solution contributes its own subtree cost (Appendix~\ref{app:theory}).

\begin{theorem}[Optimal ordering under the same restrictions]
\label{thm:optimal_guidance}
If $W>0$ and policies have equal query overhead, then
\begin{equation}
\arg\min_\pi\mathbb{E}[C_\pi(s)]
=\arg\max_\pi p_s(\pi).
\label{eq:optimal_equiv}
\end{equation}
With unequal overhead $\tau_\pi$, the quantity to minimize is $\tau_\pi+(1-p_s(\pi))W$.
\end{theorem}

Since $G$ and $W$ are fixed, increasing $p_s(\pi)$ decreases expected search cost by $W$ per unit increase. In common cost units, guidance pays off when this saving exceeds its additional query cost.

In Choco, a failed branch can change the future variable order through dom/wdeg. For a binary continuation that exhausts the failing child before switching, let $G_{\rm after}$ be the cost of the solution search after that failure and $G$ its cost when tried first. The difference between the two orders is then
\begin{equation*}
W+G_{\rm after}-G.
\end{equation*}
The adjustment $G_{\rm after}-G$ can have either sign. The labels $Q_{\mathcal S}$ account for this adjustment by measuring the total continuation from each first choice.

Let $T_0$ be native wall-clock solve time and $T_\theta$ guided wall-clock solve time under the same timing boundary, covering state capture, communication, inference, and solver-side overhead. The deployment objective is $T_\theta<T_0$, and DiBS is built around it: the query gate targets puzzles where expected search savings exceed the full cost of a call.

\section{Experiments and Analyses}
\label{sec:experiments}

\begin{table*}[t]
\centering
\caption{Complete solving of 1,500 forum-hard Sudoku instances. Each native solver is followed by its DiBS-guided variant; groups are ordered by native mean full-set time.}
\label{tab:main_forumhard}
\small
\setlength{\tabcolsep}{2.25pt}
\renewcommand{\arraystretch}{1.08}
\begin{tabular}{@{}lcccccccccccc@{}}
\toprule
\textbf{Solver} &
\multicolumn{3}{c}{\textbf{Time (s)}} &
\multicolumn{3}{c}{\textbf{Search events / puzzle}} &
\multicolumn{3}{c}{\textbf{Backtracks / puzzle}} &
\multicolumn{3}{c}{\textbf{DiBS savings (\%)}} \\
\cmidrule(lr){2-4}\cmidrule(lr){5-7}\cmidrule(lr){8-10}\cmidrule(lr){11-13}
& \textbf{Mean} & \textbf{Med} & \textbf{p95}
& \textbf{Mean} & \textbf{Med} & \textbf{p95}
& \textbf{Mean} & \textbf{Med} & \textbf{p95}
& \textbf{Events} & \textbf{Bkts} & \textbf{Time} \\
\midrule
OR-Tools CP-SAT~\citeyearpar{perron2023cpsat}
& 23.75 & 23.70 & 24.03 & 1965.20 & 2062.00 & \textbf{2286}
& 1174.35 & 1219.50 & \textbf{1334}
& \multirow{2}{*}{7.30} & \multirow{2}{*}{5.90} & \multirow{2}{*}{0.13} \\
\textbf{CP-SAT + DiBS}
& \textbf{23.72} & \textbf{23.69} & \textbf{23.91} & \textbf{1821.65} & \textbf{2038.50} & 2402
& \textbf{1105.12} & \textbf{1210.00} & 1347 & & & \\
\midrule
ACE~\citeyearpar{lecoutre2023ace}
& 18.31 & 18.27 & 19.20 & 129.60 & 104 & 324
& 21.16 & 17 & 55
& \multirow{2}{*}{4.79} & \multirow{2}{*}{5.54} & \multirow{2}{*}{0.47} \\
\textbf{ACE + DiBS}
& \textbf{18.23} & \textbf{18.22} & \textbf{18.93} & \textbf{123.39} & \textbf{102} & \textbf{311}
& \textbf{19.99} & \textbf{16} & \textbf{52} & & & \\
\midrule
JaCoP~\citeyearpar{kuchcinski2013jacop}
& 12.18 & 12.12 & 12.64 & 1242.21 & 754 & 3913
& 612.92 & 369 & 1947
& \multirow{2}{*}{20.43} & \multirow{2}{*}{20.67} & \multirow{2}{*}{15.79} \\
\textbf{JaCoP + DiBS}
& \textbf{10.25} & \textbf{10.22} & \textbf{10.62} & \textbf{988.49} & \textbf{604.50} & \textbf{3215}
& \textbf{486.20} & \textbf{294} & \textbf{1597} & & & \\
\midrule
Choco 6.0.2 LCG~\citeyearpar{prudhomme2022choco}
& 6.49 & 6.44 & 6.94 & 50.82 & 41 & 130
& 46.63 & 38 & 125
& \multirow{2}{*}{12.55} & \multirow{2}{*}{13.81} & \multirow{2}{*}{6.95} \\
\textbf{Choco LCG + DiBS}
& \textbf{6.04} & \textbf{6.01} & \textbf{6.43} & \textbf{44.44} & \textbf{35} & \textbf{112}
& \textbf{40.19} & \textbf{31} & \textbf{109} & & & \\
\midrule
Chuffed~\citeyearpar{chu2023chuffed}
& 5.83 & 5.80 & 6.12 & 54.02 & 48 & 115
& 47.92 & 43 & 108
& \multirow{2}{*}{11.36} & \multirow{2}{*}{12.44} & \multirow{2}{*}{1.51} \\
\textbf{Chuffed + DiBS}
& \textbf{5.74} & \textbf{5.71} & \textbf{6.03} & \textbf{47.89} & \textbf{41} & \textbf{108}
& \textbf{41.96} & \textbf{36} & \textbf{99} & & & \\
\midrule
Choco 6.0.2~\citeyearpar{prudhomme2022choco}
& 4.78 & 4.77 & 5.17 & 42.30 & 37 & 95
& 72.63 & 63 & 180
& \multirow{2}{*}{11.72} & \multirow{2}{*}{13.70} & \multirow{2}{*}{6.00} \\
\textbf{Choco + DiBS}
& \textbf{4.50} & \textbf{4.50} & \textbf{4.89} & \textbf{37.34} & \textbf{33} & \textbf{84}
& \textbf{62.68} & \textbf{54} & \textbf{158} & & & \\
\midrule
Gecode~\citeyearpar{schulte2026gecode}
& 3.71 & 3.71 & 3.76 & 180.26 & 156 & 437
& 88.71 & 76 & 218
& \multirow{2}{*}{16.17} & \multirow{2}{*}{16.57} & \multirow{2}{*}{7.99} \\
\textbf{Gecode + DiBS}
& \textbf{3.41} & \textbf{3.41} & \textbf{3.46} & \textbf{151.11} & \textbf{123} & \textbf{397}
& \textbf{74.01} & \textbf{60} & \textbf{198} & & & \\
\bottomrule
\end{tabular}
\end{table*}

\subsection{Experiment Settings}
\textbf{Hardware.} The comparison uses a server with two Intel Xeon Gold 6278C CPUs and one NVIDIA RTX 3090 GPU. Every benchmark process and its children are pinned to the same fixed physical CPU cores and sibling threads. Guided methods use one dedicated GPU; the CPU-only solvers use the same pinned cores.

\textbf{Model Configuration.} The frozen deployment student has 393,287 parameters, width 64, three shared-state blocks, and three action-conditioned trace blocks. It scores legal root values in one masked pass. All candidates at each root share the same state encoding. Choco and its LCG variant use a puzzle-side gate; the other adapters use solver-specific call policies. All adapters batch up to 128 states. The Choco backbone uses 27 all-different constraints with arc consistency and dom/wdeg variable selection.

\textbf{Baselines and Metrics.} We compare Choco and its LCG variant~\cite{prudhomme2022choco}, OR-Tools CP-SAT~\cite{perron2023cpsat}, Gecode~\cite{schulte2026gecode}, Chuffed~\cite{chu2023chuffed}, JaCoP~\cite{kuchcinski2013jacop}, and ACE~\cite{lecoutre2023ace} with a DiBS-guided variant of each on the same 1,500 forum-hard instances. DiBS changes root value order for Choco, its LCG variant, Gecode, JaCoP, and ACE; CP-SAT and Chuffed receive two ranked hints. Choco, Choco LCG, CP-SAT, and Chuffed query the model for 460 of 1,500 puzzles; Gecode, JaCoP, and ACE generate advice for all 1,500. Each native/guided pair uses the same solver adapter. ACE uses a cached class index and a resident worker; Chuffed uses a loaded parent with a fork per puzzle. All solutions are checked. Methods run serially on fixed CPU cores, using one dedicated GPU for DiBS, over 140 rounds of balanced cyclic and reversed orders. One-time loading is excluded; full-set time includes model construction, advice generation, communication, inference, and search.

\subsection{Main Results}

\textbf{DiBS reduces search in all seven solvers.} Table~\ref{tab:main_forumhard} uses the same Choco-trained student for all seven solvers. Bold marks the lower value within each solver pair. Search events count nodes for CP solvers and branches for CP-SAT. Savings are relative to the native solver. Mean search events fall by 11.72\% in Choco, 16.17\% in Gecode, 20.43\% in JaCoP, 12.55\% in Choco LCG, 4.79\% in ACE, and 11.36\% in Chuffed; CP-SAT branches fall by 7.30\%. Backtracks fall in every pair. Including the full cost of guidance, mean time improves by 15.79\% for JaCoP, 7.99\% for Gecode, 6.95\% for Choco LCG, and 6.00\% for Choco. Gecode is already the fastest native solver here, and DiBS reduces its time from 3.71 to 3.41 seconds. ACE, CP-SAT, and Chuffed complete the sweep at 0.47\%, 0.13\%, and 1.51\%.

\textbf{The shared decision explains the transfer.} These solvers use different propagation and branching algorithms, and each has a point where global puzzle information prioritizes a legal digit. Matching DiBS to that point lets a Choco-trained cost signal reduce search elsewhere. Direct control of a selected value converts search savings into time. Hint interfaces leave value control with the solver, so a 7.30\% branch reduction in CP-SAT and an 11.36\% node reduction in Chuffed translate into 0.13\% and 1.51\% time gains. The percentile results also show where the mean reductions are uneven. JaCoP's p95 node count falls from 3,913 to 3,215. In CP-SAT, the median branch count falls from 2,062 to 2,038.5, while p95 rises from 2,286 to 2,402.

\begin{table*}[t]
\centering
\caption{Cross-solver ablation on 1,500 puzzles. Entries are percentage changes from full DiBS at batch 128; red/green means worse/better than full DiBS. Bold marks the largest absolute change per solver and metric. Time is averaged over 144 balanced runs.}
\label{tab:ablation}
\footnotesize
\setlength{\tabcolsep}{6pt}
\renewcommand{\arraystretch}{1.0}
\begin{tabular}{lrrrrrrr}
\toprule
\textbf{Variant} & \textbf{Choco} & \textbf{Gecode} & \textbf{LCG} & \textbf{JaCoP} & \textbf{ACE} & \textbf{CP-SAT} & \textbf{Chuffed} \\
\midrule
\multicolumn{8}{l}{\textit{Solve time (\%)}} \\
\midrule
No solution pretraining & \ablworse{$+0.6$} & \ablworse{$+0.1$} & \ablbetter{$-1.3$} & \ablbetter{$-0.1$} & \ablbetter{$-0.1$} & \ablbetter{$-0.1$} & \ablbetter{$-0.3$} \\
No footprint denoising & \ablworse{$+1.3$} & \ablbetter{$-0.3$} & \ablbetter{$-0.3$} & \ablbetter{$-2.2$} & \ablworse{\textbf{+1.5}} & \ablworse{$+0.1$} & \ablbetter{$-0.1$} \\
No teacher distillation & \ablworse{$+0.8$} & \ablworse{$+1.2$} & \ablworse{$+0.7$} & \ablworse{$+1.2$} & \ablworse{$+0.3$} & \ablworse{$+0.1$} & \ablworse{$+0.1$} \\
No action ranking loss & \ablworse{$+1.2$} & \ablworse{$+0.6$} & \ablworse{$+0.3$} & \ablbetter{$-1.6$} & \ablworse{$+0.3$} & \ablbetter{$-0.1$} & $0.0$ \\
No solution-leading loss & \ablworse{$+1.8$} & \ablworse{$+0.8$} & \ablworse{$+1.0$} & \ablbetter{$-0.4$} & \ablworse{$+0.2$} & \ablworse{\textbf{+0.2}} & \ablbetter{$-0.1$} \\
Local domains only & \ablworse{$+0.7$} & \ablworse{$+1.0$} & \ablworse{$+0.4$} & \ablbetter{$-0.8$} & \ablworse{$+1.1$} & \ablworse{\textbf{+0.2}} & \ablworse{$+0.1$} \\
Toggle gate & \ablworse{\textbf{+2.7}} & \ablworse{\textbf{+4.8}} & \ablworse{\textbf{+4.1}} & \ablworse{\textbf{+5.9}} & \ablworse{$+0.3$} & \ablbetter{\textbf{-0.2}} & \ablworse{\textbf{+3.4}} \\
\midrule
\multicolumn{8}{l}{\textit{Search events (\%)}} \\
\midrule
No solution pretraining & \ablworse{$+0.2$} & \ablbetter{$-0.1$} & \ablbetter{$-0.6$} & \ablworse{$+0.1$} & \ablbetter{$-0.6$} & \ablbetter{$-0.4$} & \ablbetter{$-1.5$} \\
No footprint denoising & \ablworse{$+1.0$} & \ablbetter{$-0.7$} & \ablworse{$+0.7$} & \ablbetter{$-2.5$} & \ablworse{$+3.0$} & \ablworse{$+0.3$} & \ablbetter{$-0.6$} \\
No teacher distillation & \ablworse{$+0.7$} & \ablworse{$+0.9$} & \ablworse{$+1.2$} & \ablworse{$+1.9$} & \ablworse{$+2.2$} & \ablworse{$+0.1$} & \ablbetter{$-0.1$} \\
No action ranking loss & \ablworse{$+1.2$} & \ablworse{$+0.6$} & \ablworse{$+1.3$} & \ablbetter{$-1.6$} & \ablworse{$+0.6$} & \ablbetter{$-0.2$} & \ablbetter{$-0.5$} \\
No solution-leading loss & \ablworse{$+1.0$} & \ablworse{$+0.9$} & \ablworse{$+1.4$} & \ablbetter{$-0.2$} & \ablbetter{$-0.1$} & \ablworse{$+1.0$} & \ablworse{$+0.1$} \\
Local domains only & \ablworse{$+0.4$} & \ablbetter{$-0.4$} & \ablworse{$+1.0$} & \ablbetter{$-1.2$} & \ablworse{$+1.7$} & \ablworse{$+0.1$} & \ablbetter{$-0.6$} \\
Toggle gate & \ablbetter{\textbf{-8.7}} & \ablworse{\textbf{+6.7}} & \ablbetter{\textbf{-9.4}} & \ablworse{\textbf{+7.6}} & \ablworse{\textbf{+3.2}} & \ablbetter{\textbf{-19.8}} & \ablbetter{\textbf{-13.2}} \\
\midrule
\multicolumn{8}{l}{\textit{Backtracks (\%)}} \\
\midrule
No solution pretraining & \ablworse{$+0.2$} & \ablbetter{$-0.1$} & \ablbetter{$-0.6$} & \ablworse{$+0.1$} & \ablbetter{$-0.6$} & \ablbetter{$-0.4$} & \ablbetter{$-1.6$} \\
No footprint denoising & \ablworse{$+1.3$} & \ablbetter{$-0.7$} & \ablworse{$+0.8$} & \ablbetter{$-2.6$} & \ablworse{$+3.1$} & \ablworse{$+0.2$} & \ablbetter{$-0.6$} \\
No teacher distillation & \ablworse{$+0.8$} & \ablworse{$+0.9$} & \ablworse{$+1.3$} & \ablworse{$+1.9$} & \ablworse{$+2.7$} & \ablworse{$+0.1$} & \ablbetter{$-0.1$} \\
No action ranking loss & \ablworse{$+1.4$} & \ablworse{$+0.6$} & \ablworse{$+1.5$} & \ablbetter{$-1.6$} & \ablworse{$+0.8$} & \ablbetter{$-0.2$} & \ablbetter{$-0.7$} \\
No solution-leading loss & \ablworse{$+1.2$} & \ablworse{$+0.9$} & \ablworse{$+1.5$} & \ablbetter{$-0.2$} & \ablworse{$+0.3$} & \ablworse{$+0.8$} & $0.0$ \\
Local domains only & \ablworse{$+0.4$} & \ablbetter{$-0.4$} & \ablworse{$+1.1$} & \ablbetter{$-1.2$} & \ablworse{$+1.5$} & \ablworse{$+0.1$} & \ablbetter{$-0.7$} \\
Toggle gate & \ablbetter{\textbf{-10.5}} & \ablworse{\textbf{+6.9}} & \ablbetter{\textbf{-10.7}} & \ablworse{\textbf{+7.7}} & \ablworse{\textbf{+4.0}} & \ablbetter{\textbf{-15.2}} & \ablbetter{\textbf{-15.2}} \\
\bottomrule
\end{tabular}
\end{table*}

\subsection{Ablation Experiment}
\textbf{The design improves search across solver backbones.} Table~\ref{tab:ablation} reports changes in full-set time, search events, and backtracks relative to full DiBS. The toggle removes the gate for Choco, LCG, CP-SAT, and Chuffed and adds it for Gecode, JaCoP, and ACE. Each of five training-component removals increases search events in at least four of seven solvers, and four of them increase mean time in a majority of solvers. Distillation is especially transferable: without it, all seven solvers take longer and six explore more search events. Removing distillation raises search events and backtracks by 2.2\% and 2.7\% in ACE, the largest increases for that ablation. Changing the query gate slows six solvers. For Choco, removing the gate reduces search events by 8.7\% and backtracks by 10.5\%, while full-set time rises by 2.7\%. This comparison pairs the work performed by the solver with the measured cost of querying the model on more puzzles.

\textbf{Solver behavior explains the different contributions.} The pretraining ablation shows that search-cost supervision carries the action-choice signal; solution reconstruction adds a small, solver-dependent adjustment. Footprint denoising helps ACE but is less aligned with JaCoP's continuation: removing it raises ACE search events by 3.0\% while reducing JaCoP events by 2.5\%. Routing resolves a second tradeoff. Calling the model more often reduces search events in Choco, LCG, CP-SAT, and Chuffed, while CP-SAT also recovers the added solve time; Gecode and JaCoP need frequent guidance, so the gate increases both search and time. The components therefore serve distinct roles: training improves action choice, while batching and solver-aware routing turn that choice into faster complete search.

\section{Discussion}

\subsection{DiBS vs.\ Puzzle Difficulty}
\begin{figure}[t]
\centering
\includegraphics[width=1\linewidth]{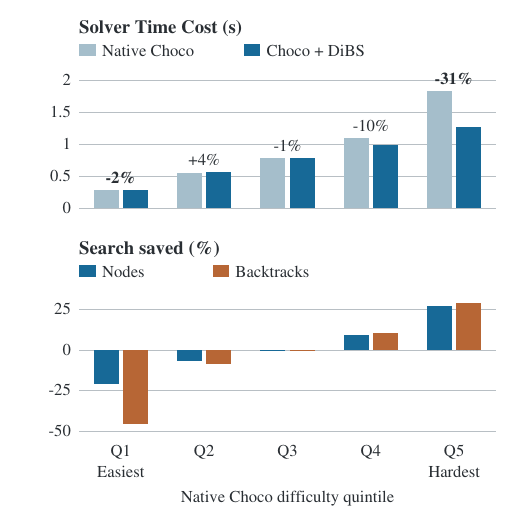}
\caption{DiBS by difficulty quintile (300 puzzles each, ordered by native Choco nodes). Top: solver time averaged over 140 rounds, excluding model and batch overhead. Bottom: relative node and backtrack savings.}
\label{fig:discussion_difficulty}
\vspace{-8pt}
\end{figure}

\textbf{The benefit is concentrated in difficult search.} We rank the 1,500 test puzzles by the nodes native Choco explores, divide them into five equal groups, and compare each puzzle with its DiBS-guided run using the frozen Table~\ref{tab:main_forumhard} student. In the hardest group, DiBS removes 7,076 of 25,965 nodes (27.25\%) and 14,189 of 48,234 backtracks (29.42\%). The next-hardest group saves 9.40\% of nodes. These two groups account for the gain: the hardest 300 puzzles alone provide 95\% of the net 7,434 nodes saved across all 1,500 puzzles. Solver time falls by 30.7\% in the hardest group and 2.0\% in the easiest group (Figure~\ref{fig:discussion_difficulty}).

\textbf{Easy puzzles define the operating regime.} Native Choco needs only 3--17 nodes on each puzzle in the easiest group. There DiBS adds 681 nodes in total, 2.3 per puzzle; the middle group stays flat. A nearly solved search has little wrong-branch work to remove, and a changed value order disturbs an already good path. The gate captures this asymmetry: it queries 160 of the hardest 300 puzzles and 75 of the easiest 300, so the residual easy-group cost is 2.3 nodes per puzzle against a 7,076-node saving in the hardest group. Clue count does not track the saving: 23-clue puzzles save 17.36\% of nodes, whereas 25-clue puzzles save 7.69\%. The solver's actual search burden is the relevant quantity.

\textbf{This pattern connects the empirical gain to the cost model.} In the restricted binary case of Theorem~\ref{thm:binary_cost}, choosing the useful branch first avoids the failing-subtree cost $W$; a larger $W$ creates more room for guidance. The adaptive bound in Theorem~\ref{thm:cost_prediction} relates errors in predicted action costs to the nodes explored after the chosen action. The grouped counts identify where search savings arise; Table~\ref{tab:main_forumhard} shows that they survive the full timing boundary, including model calls and communication.

\subsection{DiBS vs.\ Denoising Steps}
\begin{figure}[t]
\centering
\includegraphics[width=1\linewidth]{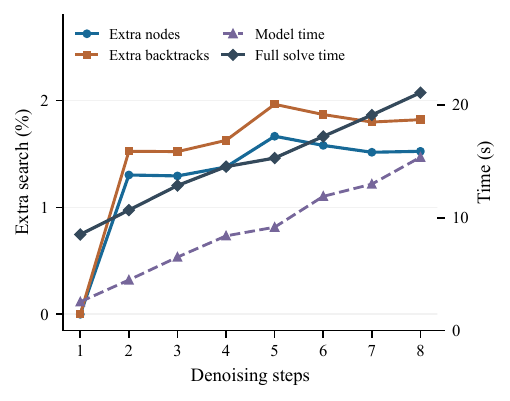}
\caption{Effect of denoising steps on 1,500 puzzles solved one at a time. Extra nodes and backtracks use the left axis; model-and-transfer time and complete solve time use the right axis.}
\label{fig:discussion_steps}
\vspace{-8pt}
\end{figure}

\textbf{We varied the number of denoising steps.} The same student, gate, and 1,500 puzzles are used at every setting from one through eight denoising steps. Each puzzle is solved individually so that every model query contributes to its solve time. A step recomputes branch scores from the current masked trace; additional steps let the model revise them before Choco receives the value order. Figure~\ref{fig:discussion_steps} shows one complete test-set pass at each setting.

\textbf{One step uses the fewest nodes and backtracks.} One step explores 56,009 nodes and 94,022 backtracks. Two steps increase those counts to 56,739 and 95,456; all settings from two to eight remain above the one-step counts. The largest node count, 56,942, occurs at five steps; every setting above one step underperforms one step. Across the multistep settings, the node count is 1--2\% higher than with one step.

\textbf{Runtime increases with the number of denoising steps.} Model-and-transfer time grows from 2.53 s at one step to 15.37 s at eight; complete solve time rises from 8.51 s to 21.10 s, increasing at every step. Additional steps also explore more nodes and backtracks than one step. The inputs to later steps may contribute to this pattern. Training corrupts true search footprints, and later inference steps condition on the model's estimates. Deployment uses one fully masked pass.

\subsection{DiBS Beyond Sudoku: SAT}
\label{sec:sat_transfer}
\begin{table*}[t]
\centering
\caption{DiBS on SAT. Each solver is paired with its guided variant on 500 hard formulas. Statistics are averaged over 24 rounds.}
\label{tab:sat_transfer}
\small
\setlength{\tabcolsep}{2pt}
\renewcommand{\arraystretch}{1.08}
\begin{tabular*}{\textwidth}{@{\extracolsep{\fill}}cccccccccccc@{}}
\toprule
\multirow{3}{*}[-0.8ex]{\textbf{Solver}} & \multirow{3}{*}[-0.8ex]{\textbf{Solved}} & \multicolumn{7}{c}{\textbf{Commonly solved instances}} & \multicolumn{3}{c}{\textbf{Independently solved}} \\
\cmidrule(lr){3-9}\cmidrule(lr){10-12}
& & \multicolumn{3}{c}{\textbf{Nodes}} & \multicolumn{3}{c}{\textbf{Backtracks}} & \textbf{Time} & \textbf{Nodes} & \textbf{Backtracks} & \textbf{Time} \\
\cmidrule(lr){3-5}\cmidrule(lr){6-8}
& & \textbf{Mean} & \textbf{Med} & \textbf{p95} & \textbf{Mean} & \textbf{Med} & \textbf{p95} & \textbf{Mean} & \textbf{Mean} & \textbf{Mean} & \textbf{Mean} \\
\midrule
Choco~\citeyearpar{prudhomme2022choco} & 16.00 & 4.60 & 4.48 & 9.07 & 9.17 & 8.93 & 18.09 & 0.38 & 5.31 & 10.58 & 0.44 \\
\textbf{Choco + DiBS} & \textbf{24.71} & \textbf{2.76} & \textbf{1.78} & \textbf{7.07} & \textbf{5.47} & \textbf{3.52} & \textbf{14.09} & \textbf{0.23} & \textbf{3.26} & \textbf{6.43} & \textbf{0.29} \\
\midrule
Choco LCG~\citeyearpar{prudhomme2022choco} & 16.21 & 3.58 & 4.11 & \textbf{6.23} & 3.57 & 4.10 & \textbf{6.21} & 0.37 & 3.88 & 3.87 & 0.41 \\
\textbf{Choco LCG + DiBS} & \textbf{31.92} & \textbf{2.33} & \textbf{1.56} & 6.31 & \textbf{2.31} & \textbf{1.54} & 6.29 & \textbf{0.24} & \textbf{2.99} & \textbf{2.95} & \textbf{0.34} \\
\midrule
Gecode~\citeyearpar{schulte2026gecode} & 88.00 & 27.58 & 30.99 & \textbf{43.99} & 13.79 & 15.49 & \textbf{21.99} & 0.45 & 27.39 & 13.69 & 0.46 \\
\textbf{Gecode + DiBS} & \textbf{130.79} & \textbf{24.85} & \textbf{28.97} & 45.54 & \textbf{12.42} & \textbf{14.48} & 22.77 & \textbf{0.42} & \textbf{19.83} & \textbf{9.91} & \textbf{0.38} \\
\midrule
JaCoP~\citeyearpar{kuchcinski2013jacop} & 15.46 & 19.11 & 13.67 & 40.87 & 9.54 & 6.82 & 20.42 & 0.46 & 18.10 & 9.03 & 0.45 \\
\textbf{JaCoP + DiBS} & \textbf{36.42} & \textbf{11.91} & \textbf{8.65} & \textbf{29.77} & \textbf{5.94} & \textbf{4.31} & \textbf{14.87} & \textbf{0.32} & \textbf{9.35} & \textbf{4.65} & \textbf{0.30} \\
\midrule
Glucose 4.2.1~\citeyearpar{audemard2009glucose} & 124.21 & 16.91 & 14.72 & 40.47 & 14.78 & 13.05 & 34.82 & 0.22 & 17.89 & 15.58 & 0.24 \\
\textbf{Glucose 4.2.1 + DiBS} & \textbf{146.00} & \textbf{14.89} & \textbf{13.29} & \textbf{39.98} & \textbf{13.05} & \textbf{11.83} & \textbf{34.34} & \textbf{0.20} & \textbf{15.67} & \textbf{13.67} & \textbf{0.23} \\
\midrule
MapleChrono~\citeyearpar{nadel2018chronological} & 19.17 & 27.13 & 29.32 & 41.59 & 21.46 & 23.25 & 33.80 & 0.57 & 31.25 & 24.73 & 0.63 \\
\textbf{MapleChrono + DiBS} & \textbf{23.88} & \textbf{18.75} & \textbf{18.82} & \textbf{37.83} & \textbf{14.58} & \textbf{14.50} & \textbf{29.77} & \textbf{0.32} & \textbf{21.13} & \textbf{15.70} & \textbf{0.37} \\
\bottomrule
\end{tabular*}
\end{table*}

\textbf{SAT uses the same division of labor.} We extend DiBS to Boolean formulas using a clause--literal graph diffusion model based on DiffusionSAT~\cite{lumii2025diffusionsat}. A teacher learns from verified satisfying assignments and paired CaDiCaL search costs; distillation transfers this signal to a 0.22M-parameter student. One denoising step supplies a global candidate assignment. Constraint solvers combine it with their current assignments to select a value, while Glucose and MapleChrono use it to initialize their preferred Boolean values. The solver retains propagation, learning where applicable, and backtracking. We use 312,300 training formulas, 5,000 validation formulas, and a fixed hard test set of 185 3-SAT and 315 4-SAT formulas. Each method runs with a one-second search budget in 24 balanced rounds under the fixed-resource protocol of Section~5.1. Table~\ref{tab:sat_transfer} reports the number solved and mean per-instance search counts and time, both on the pair's commonly solved formulas and on each method's independently solved formulas. Bold marks the better value within each solver pair. Counts are in thousands. Time is in seconds per solved formula and includes amortized model cost. Data, training, and interface details are in Appendix~\ref{app:sat_details}.

\textbf{DiBS improves all six backbones.} Mean solved counts increase by 17.54--135.58\%, including 88.00 to 130.79 for Gecode and 124.21 to 146.00 for Glucose. On commonly solved formulas, nodes or decisions decrease by 9.90--40.03\%, and backtracks or conflicts decrease by 9.91--40.30\%. These matched comparisons also reduce mean per-instance time by 5.89--43.51\%, including the model cost. DiBS therefore both reaches additional answers and spends less work on formulas that either variant can solve. The independently solved columns show that the expanded solved sets also have lower mean search counts and time in every pair.

\textbf{The global assignment is a reusable search signal.} A clause links several variables, so a globally compatible assignment can identify promising values beyond what a local branch rule sees. The interface makes that information relevant to the solver's next choice: constraint solvers update clause support using their current assignments, whereas conflict-driven solvers consume the same output as a phase preference and then adapt it through their own search. Reusing one output throughout a solve avoids repeated model calls. The same CaDiCaL-trained student serves four constraint-programming backbones and two conflict-driven SAT solvers. Gecode's larger coverage and smaller matched reduction show that reaching more formulas and shortening existing solves are complementary benefits. The transferable element is a globally informed value choice, delivered through the solver's own decision mechanism.

\textbf{Efficiency and coverage are complementary.} The common cohort compares both solvers on the same formulas. The independently solved cohorts can differ, so their means reflect both search effort and which formulas finish within budget. Gecode's mean time on the common cohort falls by 5.89\%, and its solved count rises from 88.00 to 130.79 per round under the one-second budget. For JaCoP, mean time on formulas solved by both variants falls from 0.46 to 0.32 seconds. With the same one-second search budget, the guided variant solves an average of 36.42 formulas per round, up from 15.46 with native search.

\textbf{Search savings vary across the cost distribution.} Among formulas solved by both variants, JaCoP's median node count falls from 13.67 to 8.65 thousand, and its 95th percentile falls from 40.87 to 29.77 thousand. Gecode's median falls from 30.99 to 28.97 thousand, while its 95th percentile rises from 43.99 to 45.54 thousand. Choco LCG's median node count falls from 4.11 to 1.56 thousand. Its 95th percentile rises slightly, from 6.23 to 6.31 thousand. The matched set is formed separately in each round, and Table~\ref{tab:sat_transfer} averages the resulting percentiles over 24 rounds. For the two conflict-driven solvers, the counters are decisions and conflicts. MapleChrono's mean conflict count falls from 21.46 to 14.58 thousand, with mean time falling from 0.57 to 0.32 seconds per commonly solved formula. The time comparison includes the allocated cost of batched guidance.

\section{Conclusion}

DiBS uses diffusion-informed guidance to order values in complete search. A compact student trained on Choco counterfactual action costs reduces search in seven solver backbones, and its batched interface lowers full-set time in every pair, with the largest gains on hard instances where wrong branches are expensive. The theory connects prediction error to action cost and identifies the exact regime in which the wrong-subtree identity yields its saving. The SAT extension carries the same principle to six constraint-programming and conflict-driven backbones, where one CaDiCaL-trained graph student increases solved counts and reduces mean search effort. These results establish learned global guidance as a practical, low-overhead layer for complete search; a natural next step is to refresh guidance when the expected search saving justifies another call.

\clearpage
{\small
\bibliography{references}

@article{mackworth1977consistency,
  title   = {Consistency in Networks of Relations},
  author  = {Mackworth, Alan K.},
  journal = {Artificial Intelligence},
  volume  = {8},
  number  = {1},
  pages   = {99--118},
  year    = {1977},
  doi     = {10.1016/0004-3702(77)90007-8}
}

@article{haralick1980increasing,
  title   = {Increasing Tree Search Efficiency for Constraint Satisfaction Problems},
  author  = {Haralick, Robert M. and Elliott, Gordon L.},
  journal = {Artificial Intelligence},
  volume  = {14},
  number  = {3},
  pages   = {263--313},
  year    = {1980},
  doi     = {10.1016/0004-3702(80)90051-X}
}

@inproceedings{regin1994filtering,
  title     = {A Filtering Algorithm for Constraints of Difference in {CSP}s},
  author    = {R{\'e}gin, Jean-Charles},
  booktitle = {Proceedings of the Twelfth National Conference on Artificial Intelligence},
  pages     = {362--367},
  year      = {1994}
}

@inproceedings{simonis2005sudoku,
  title     = {Sudoku as a Constraint Problem},
  author    = {Simonis, Helmut},
  booktitle = {Proceedings of the 4th International Workshop on Modelling and Reformulating Constraint Satisfaction Problems},
  pages     = {13--27},
  year      = {2005}
}

@inproceedings{palm2018recurrent,
  title     = {Recurrent Relational Networks},
  author    = {Palm, Rasmus Berg and Paquet, Ulrich and Winther, Ole and Bengio, Yoshua},
  booktitle = {Advances in Neural Information Processing Systems},
  volume    = {31},
  year      = {2018}
}

@inproceedings{wang2019satnet,
  title     = {{SATNet}: Bridging Deep Learning and Logical Reasoning Using a Differentiable Satisfiability Solver},
  author    = {Wang, Po-Wei and Donti, Priya L. and Wilder, Bryan and Kolter, J. Zico},
  booktitle = {Proceedings of the 36th International Conference on Machine Learning},
  series    = {Proceedings of Machine Learning Research},
  volume    = {97},
  pages     = {6545--6554},
  year      = {2019}
}

@inproceedings{khalil2016learning,
  title     = {Learning to Branch in Mixed Integer Programming},
  author    = {Khalil, Elias B. and Le Bodic, Pierre and Song, Le and Nemhauser, George and Dilkina, Bistra},
  booktitle = {Proceedings of the Thirtieth AAAI Conference on Artificial Intelligence},
  pages     = {724--731},
  year      = {2016}
}

@inproceedings{gasse2019exact,
  title     = {Exact Combinatorial Optimization with Graph Convolutional Neural Networks},
  author    = {Gasse, Maxime and Ch{\'e}telat, Didier and Ferroni, Nicola and Charlin, Laurent and Lodi, Andrea},
  booktitle = {Advances in Neural Information Processing Systems},
  volume    = {32},
  year      = {2019}
}

@inproceedings{ho2020denoising,
  title     = {Denoising Diffusion Probabilistic Models},
  author    = {Ho, Jonathan and Jain, Ajay and Abbeel, Pieter},
  booktitle = {Advances in Neural Information Processing Systems},
  volume    = {33},
  pages     = {6840--6851},
  year      = {2020}
}

@inproceedings{austin2021structured,
  title     = {Structured Denoising Diffusion Models in Discrete State-Spaces},
  author    = {Austin, Jacob and Johnson, Daniel D. and Ho, Jonathan and Tarlow, Daniel and van den Berg, Rianne},
  booktitle = {Advances in Neural Information Processing Systems},
  volume    = {34},
  pages     = {17981--17993},
  year      = {2021}
}

@inproceedings{weilbach2023graphically,
  title     = {Graphically Structured Diffusion Models},
  author    = {Weilbach, Christian D. and Harvey, William and Wood, Frank},
  booktitle = {Proceedings of the 40th International Conference on Machine Learning},
  series    = {Proceedings of Machine Learning Research},
  volume    = {202},
  pages     = {37659--37678},
  year      = {2023}
}

@inproceedings{sun2023difusco,
  title     = {{DIFUSCO}: Graph-based Diffusion Solvers for Combinatorial Optimization},
  author    = {Sun, Zhiqing and Yang, Yiming},
  booktitle = {Advances in Neural Information Processing Systems},
  volume    = {36},
  pages     = {3706--3731},
  year      = {2023}
}

@inproceedings{ye2025beyond,
  title     = {Beyond Autoregression: Discrete Diffusion for Complex Reasoning and Planning},
  author    = {Ye, Jiacheng and Gao, Jiahui and Gong, Shansan and Zheng, Lin and Jiang, Xin and Li, Zhenguo and Kong, Lingpeng},
  booktitle = {International Conference on Learning Representations},
  year      = {2025}
}

@inproceedings{janner2022diffuser,
  title     = {Planning with Diffusion for Flexible Behavior Synthesis},
  author    = {Janner, Michael and Du, Yilun and Tenenbaum, Joshua B. and Levine, Sergey},
  booktitle = {Proceedings of the 39th International Conference on Machine Learning},
  series    = {Proceedings of Machine Learning Research},
  volume    = {162},
  pages     = {9902--9915},
  year      = {2022}
}

@article{campbell2022ctdd,
  title={A Continuous Time Framework for Discrete Denoising Models},
  author={Campbell, Andrew and Benton, Joe and De Bortoli, Valentin and Rainforth, Tom and Deligiannidis, George and Doucet, Arnaud},
  journal={arXiv preprint arXiv:2205.14987},
  year={2022}
}

@article{freuder1982sufficient,
  title={A Sufficient Condition for Backtrack-Free Search},
  author={Freuder, Eugene C.},
  journal={Journal of the ACM},
  volume={29},
  number={1},
  pages={24--32},
  year={1982}
}

@article{dechter1992backjumping,
  title={From Backtracking to Backjumping: Learning from Failures in Constraint Satisfaction},
  author={Dechter, Rina},
  journal={Artificial Intelligence},
  volume={57},
  number={2-3},
  pages={227--253},
  year={1992}
}

@book{apt2003principles,
  title={Principles of Constraint Programming},
  author={Apt, Krzysztof R.},
  publisher={Cambridge University Press},
  year={2003}
}

@book{dechter2003constraint,
  title={Constraint Processing},
  author={Dechter, Rina},
  publisher={Morgan Kaufmann},
  year={2003}
}

@inproceedings{scavuzzo2022treemdp,
  title={Learning to Branch with Tree MDPs},
  author={Scavuzzo, Luca and Chen, Didier and Beinke, Daniel and Lodi, Andrea and Bengio, Yoshua and Prouvost, Antoine},
  booktitle={NeurIPS},
  year={2022}
}

@article{meng2021sudoku,
  title={Solving Sudoku with Neural Networks},
  author={Meng, Qi and Chang, Hsieh},
  journal={International Conference on Machine Learning},
  year={2021}
}

@inproceedings{harvey1995limited,
  author    = {William D. Harvey and Matthew L. Ginsberg},
  title     = {Limited Discrepancy Search},
  booktitle = {Proceedings of the 14th International Joint Conference on Artificial Intelligence (IJCAI)},
  pages     = {607--613},
  year      = {1995}
}

@inproceedings{walsh1997depth,
  author    = {Toby Walsh},
  title     = {Depth-Bounded Discrepancy Search},
  booktitle = {Proceedings of the 15th International Joint Conference on Artificial Intelligence (IJCAI)},
  pages     = {1388--1395},
  year      = {1997}
}

@inproceedings{audemard2009glucose,
  title={Predicting Learnt Clauses Quality in Modern SAT Solvers},
  author={Audemard, Gilles and Simon, Laurent},
  booktitle={Proceedings of the 21st International Joint Conference on Artificial Intelligence},
  pages={399--404},
  year={2009}
}

@inproceedings{nadel2018chronological,
  title={Chronological Backtracking},
  author={Nadel, Alexander and Ryvchin, Vadim},
  booktitle={Theory and Applications of Satisfiability Testing -- SAT 2018},
  year={2018}
}

@misc{lumii2025diffusionsat,
  author={{LUMII Syslab}},
  title={{DiffusionSAT}: Official Implementation},
  year={2025},
  howpublished={\url{https://github.com/LUMII-Syslab/DiffusionSAT}},
  note={Revision 2258e4873c2be38732463edaf93a6c43fb89ac6f}
}

@article{prudhomme2022choco,
  title={Choco-solver: A Java library for constraint programming},
  author={Prud'homme, Charles and Fages, Jean-Guillaume},
  journal={Journal of Open Source Software},
  volume={7},
  number={78},
  pages={4708},
  year={2022},
  doi={10.21105/joss.04708}
}

@manual{schulte2026gecode,
  title={Modeling and Programming with Gecode},
  author={Schulte, Christian and Tack, Guido and Lagerkvist, Mikael Z.},
  year={2026},
  note={Gecode 6.4.0 documentation},
  url={https://www.gecode.dev/doc/latest/modeling/}
}

@misc{chu2023chuffed,
  title={The Chuffed Solver},
  author={Chu, Geoffrey and Stuckey, Peter J. and Schutt, Andreas and Ehlers, Thorsten and Gange, Graeme and Francis, Kathryn},
  year={2023},
  howpublished={\url{https://github.com/chuffed/chuffed}}
}

@article{lecoutre2023ace,
  title={ACE, a Generic Constraint Solver},
  author={Lecoutre, Christophe},
  journal={arXiv preprint arXiv:2302.05405},
  year={2023},
  url={https://arxiv.org/abs/2302.05405}
}

@inproceedings{perron2023cpsat,
  title={The CP-SAT-LP Solver},
  author={Perron, Laurent and Didier, Fr\'{e}d\'{e}ric and Gay, Steven},
  booktitle={Proceedings of the 29th International Conference on Principles and Practice of Constraint Programming},
  series={LIPIcs},
  volume={280},
  pages={3:1--3:2},
  year={2023},
  doi={10.4230/LIPIcs.CP.2023.3}
}

@inproceedings{kuchcinski2013jacop,
  title={JaCoP: Java Constraint Programming Solver},
  author={Kuchcinski, Krzysztof and Szymanek, Radoslaw},
  booktitle={CP Solvers: Modeling, Applications, Integration, and Standardization},
  year={2013},
  url={https://portal.research.lu.se/en/publications/jacop-java-constraint-programming-solver}
}

@inproceedings{frost1995lookahead,
 title={Look-Ahead Value Ordering for Constraint Satisfaction Problems},
 author={Frost, Daniel and Dechter, Rina},
 booktitle={Proceedings of the Fourteenth International Joint Conference on Artificial Intelligence},
 pages={572--577}, year={1995}, url={https://ijcai.org/proceedings/1995/75}
}

@inproceedings{smith2005value,
 title={Value Ordering for Finding All Solutions},
 author={Smith, Barbara M. and Sturdy, Paula},
 booktitle={Proceedings of the Nineteenth International Joint Conference on Artificial Intelligence},
 pages={311--315}, year={2005}, url={https://ijcai.org/proceedings/2005/132}
}

@article{song2021variable,
 title={Learning Variable Ordering Heuristics for Solving Constraint Satisfaction Problems},
 author={Song, Wen and Cao, Zhiguang and Zhang, Jie and Lim, Andrew},
 journal={Engineering Applications of Artificial Intelligence}, volume={104}, pages={104603}, year={2021}, doi={10.1016/j.engappai.2021.104603}
}

@inproceedings{balcan2018learning,
 title={Learning to Branch}, author={Balcan, Maria-Florina and Dick, Travis and Sandholm, Tuomas and Vitercik, Ellen},
 booktitle={Proceedings of the 35th International Conference on Machine Learning}, series={Proceedings of Machine Learning Research}, volume={80}, pages={344--353}, year={2018}, url={https://proceedings.mlr.press/v80/balcan18a.html}
}

@inproceedings{tonshoff2023one,
 title={One Model, Any {CSP}: Graph Neural Networks as Fast Global Search Heuristics for Constraint Satisfaction},
 author={T\"onshoff, Jan and Kisin, Berke and Lindner, Jakob and Grohe, Martin},
 booktitle={Proceedings of the Thirty-Second International Joint Conference on Artificial Intelligence}, pages={4280--4288}, year={2023}, doi={10.24963/ijcai.2023/476}
}

@article{marty2024value,
 title={Learning and Fine-Tuning a Generic Value-Selection Heuristic Inside a Constraint Programming Solver},
 author={Marty, Tom and Boisvert, L\'eo and Fran\c{c}ois, Tristan and Tessier, Pierre and Gautier, Louis and Rousseau, Louis-Martin and Cappart, Quentin},
 journal={Constraints}, volume={29}, pages={234--260}, year={2024}, doi={10.1007/s10601-024-09377-4}
}

@inproceedings{cheng2026theoretical,
 title={Theoretical Challenges in Learning for Branch-and-Cut}, author={Cheng, Hongyu and Basu, Amitabh},
 booktitle={Proceedings of the 43rd International Conference on Machine Learning}, series={Proceedings of Machine Learning Research}, volume={306}, pages={18757--18781}, year={2026}, url={https://proceedings.mlr.press/v306/cheng26a.html}
}

@misc{jung2025guided,
 title={Guided Discrete Diffusion for Constraint Satisfaction Problems}, author={Jung, Justin}, year={2025}, eprint={2512.14765}, archivePrefix={arXiv}, url={https://arxiv.org/abs/2512.14765}
}

@inproceedings{yang2023compositional,
 title={Compositional Diffusion-Based Continuous Constraint Solvers},
 author={Yang, Zhutian and Mao, Jiayuan and Du, Yilun and Wu, Jiajun and Tenenbaum, Joshua B. and Lozano-P\'erez, Tom\'as and Kaelbling, Leslie Pack},
 booktitle={Proceedings of the 7th Conference on Robot Learning}, series={Proceedings of Machine Learning Research}, volume={229}, pages={3242--3265}, year={2023}, url={https://proceedings.mlr.press/v229/yang23d.html}
}

@inproceedings{refalo2004impact,
 title={Impact-Based Search Strategies for Constraint Programming}, author={Refalo, Philippe},
 booktitle={Principles and Practice of Constraint Programming -- CP 2004}, series={Lecture Notes in Computer Science}, volume={3258}, pages={557--571}, year={2004}, doi={10.1007/978-3-540-30201-8_41}
}

@inproceedings{zhang2008learned,
 title={Learned Value-Ordering Heuristics for Constraint Satisfaction}, author={Zhang, Zhijun and Epstein, Susan L.},
 booktitle={Proceedings of the AAAI Workshop on Search in Artificial Intelligence and Robotics}, year={2008}, pages={154--161}, url={https://cdn.aaai.org/Workshops/2008/WS-08-10/WS08-10-023.pdf}
}

@inproceedings{luchterhand2025impact,
 title={Understanding the Impact of Value Selection Heuristics in Scheduling Problems},
 author={Luchterhand, Tim and Hebrard, Emmanuel and Thi\'ebaux, Sylvie},
 booktitle={Proceedings of the 31st International Conference on Principles and Practice of Constraint Programming}, series={LIPIcs}, volume={340}, pages={27:1--27:23}, year={2025}, doi={10.4230/LIPIcs.CP.2025.27}
}

@inproceedings{chalumeau2021seapearl,
 title={SeaPearl: A Constraint Programming Solver Guided by Reinforcement Learning},
 author={Chalumeau, F\'elix and Coulon, Ilan and Cappart, Quentin and Rousseau, Louis-Martin},
 booktitle={Proceedings of the 18th International Conference on Integration of Constraint Programming, Artificial Intelligence, and Operations Research}, pages={392--409}, year={2021}, doi={10.1007/978-3-030-78230-6_25}
}
\bibliographystyle{icml2026}
}

\clearpage

﻿
\appendix
\section{Implementation Details}
\label{app:solver_adapters}
\subsection{Sudoku data and counterfactual labels}
The consolidated Sudoku database contains 4,159,329 checked training puzzle--solution pairs and 12,000 validation pairs; uniqueness is verified before a label is used. Search training selects 2,666,084 branchable roots with complete action labels, while puzzles solved by root propagation contribute only to solution pretraining. The 1,500 forum-hard puzzles form a separate test set. From each labeled root, a fresh Choco instance propagates the clues, selects a variable with its native heuristic, and tries every legal digit first from the same state with adaptive updates enabled. The continuation supplies $Q_{\mathcal S}(s,a)$; runs that reach the labeling budget are censored. Each continuation also records 81 cell-selection counts, quantized into $0$, $1$, $2$--$3$, $4$--$7$, and $8+$, giving a five-class spatial footprint of where search works, including failed branches.

\subsection{Model and training objectives}
\begin{center}
\small
\begin{tabular}{lcc}
\toprule
Configuration & Teacher & Student \\
\midrule
Feature width & 192 & 64 \\
Shared-state blocks & 8 & 3 \\
Action-conditioned blocks & 4 & 3 \\
Parameters & 7,284,295 & 393,287 \\
Search-training epochs & 8 & 8 \\
Deployment denoising passes & Offline only & 1 \\
\bottomrule
\end{tabular}
\end{center}

The teacher encoder is initialized by six epochs of solution reconstruction on the full pair database, then trained on propagated grids, legal-domain masks, clue indicators, and candidate actions with learning rates $5\times10^{-5}$ for the encoder and $5\times10^{-4}$ for the remaining parameters. Checkpoints are selected by the mean measured action cost induced by the predicted one-pass ordering on validation roots. Let $q_a=\log(1+Q_{\mathcal S}(s,a))$, let $r_a$ be its prediction, and let $y_a$ indicate whether action $a$ agrees with the verified solution. The supervised components are
\begin{align}
\mathcal L_Q&=\operatorname{mean}_a\operatorname{Huber}(r_a-q_a),\\
\mathcal L_{\mathrm{rank}}&=\operatorname{mean}_{a,b:q_a\ne q_b}
 w_{ab}\log\!\left(1+e^{-\sigma_{ab}(r_a-r_b)}\right),
\end{align}
where $\sigma_{ab}=\operatorname{sign}(q_a-q_b)$ and $w_{ab}=\min(|q_a-q_b|,2)$. A binary cross-entropy term $\mathcal L_{\mathrm{sol}}$ supervises the solution-leading output, and $\mathcal L_{\mathrm{fp}}$ averages cross-entropy over fully masked and partially corrupted footprint inputs, so the deployment condition is included in training. The teacher objective is
\begin{equation}
\mathcal L_T=\mathcal L_Q+0.25\mathcal L_{\mathrm{rank}}
 +0.2\mathcal L_{\mathrm{sol}}+0.5\mathcal L_{\mathrm{fp}}.
\end{equation}

For distillation, $p_T(a\mid s)=\exp(-r_T(s,a)/\tau)/\sum_{b\in A(s)}\exp(-r_T(s,b)/\tau)$ with $\tau=0.2$. The student keeps exact cost, ranking, and solution-leading supervision and adds $0.5\operatorname{KL}(p_T\Vert p_S)$; its footprint reconstruction weight is zero because the teacher transfers candidate relations while the student is optimized for fully masked one-pass inference. The online decision uses predicted candidate costs: lower is preferred, and the other legal values remain available to the solver.

\subsection{Solver interfaces and timing}
The value-order adapters intercept the decision after each solver selects its variable: Choco and Choco LCG after dom/wdeg, Gecode after its AFC rule, JaCoP after its native variable selector, and ACE at its root hook. Each sends the selected cell and exact legal values to the frozen model and tries the preferred value first; no adapter removes the other values. CP-SAT and Chuffed receive two ranked hints, chosen by the predicted cost gap and filtered for conflicts within a Sudoku house, through solution hints and integer warm-start annotations. A queried root travels as a 325-byte record containing the puzzle, grid, 81 domain bitmasks, and selected cell; a resident GPU model scores up to 128 states in one request. The gate queries 460 of 1,500 puzzles for Choco, Choco LCG, CP-SAT, and Chuffed, while Gecode, JaCoP, and ACE receive advice for all 1,500.

The 14 Sudoku methods occupy every position of 14 cyclic orders and their reversals; five repetitions give the 140 rounds of Table~\ref{tab:main_forumhard}. Fixed CPU affinity and a dedicated GPU control placement, and rotating the order balances cache, temperature, and frequency effects. Startup and warmup are excluded; state preparation, communication, inference, construction, and search are timed. In Table~\ref{tab:ablation}, training ablations change the model, while gate and batch ablations change deployment. Batch size affects only how model work is shared, so search counts are invariant; gate changes alter which puzzles receive advice and therefore both counts and time. Sudoku and SAT models are trained separately for their respective representations and targets. Algorithm~\ref{alg:dibs_solver} summarizes the deployed intervention.

\renewcommand{\algorithmiccomment}[1]{\hfill\textit{#1}}
\begin{algorithm}[t]
\caption{DiBS-guided complete Choco search}
\label{alg:dibs_solver}
\begin{algorithmic}[1]
\REQUIRE puzzle $I$; frozen student $r_\theta$; query gate $g$
\ENSURE complete search result
\STATE $S\gets\mathrm{NewChoco}(I)$; \quad $S.\mathrm{PropagateRoot}()$
\IF{$S$ is solved or inconsistent}
\STATE \textbf{return} $S.\mathrm{Result}()$
\ENDIF
\STATE $v\gets S.\mathrm{SelectVariable}()$ \COMMENT{dom/wdeg}
\STATE $L\gets S.\mathrm{NativeValueOrder}(D_v)$
\IF{$g(I)$ and $|D_v|>1$}
\STATE $X\gets\mathrm{CaptureState}(S)$ \COMMENT{state}
\STATE $d^\star\gets\arg\min_{d\in D_v} r_\theta(X,v,d)$ \COMMENT{one batch}
\STATE $L\gets\mathrm{Promote}(L,d^\star)$ \COMMENT{preferred first}
\ENDIF
\STATE $S.\mathrm{SetValueOrder}(v,L)$
\STATE \textbf{return} $S.\mathrm{Solve}()$ \COMMENT{complete search}
\end{algorithmic}
\end{algorithm}

\section{SAT Details}
\label{app:sat_details}
\subsection{Teacher, student, and losses}
We parse and deduplicate public 3-SAT and 4-SAT formulas from CSPBenchmarks and SATLIB; the splits contain 312,300 training and 5,000 validation formulas, and the fixed 500-formula test set has 185 3-SAT and 315 4-SAT instances with 64--256 variables. A formula can admit several satisfying assignments. The teacher learns clause--assignment compatibility across assignments. We reproduce DiffusionSAT's clause--literal graph reasoner (128 features, 32 recurrent rounds) and pretrain it on 91,863 verified satisfiable formulas with at most 100 variables. Search alignment then uses 145,409 paired CaDiCaL root-action records with denoising replay and a frozen-reference loss; the backbone and cost head use learning rates $10^{-5}$ and $10^{-4}$ for eight epochs. Candidate assignments and noisy states are cached over all 312,300 training formulas. The deployed student uses 64 features and 16 rounds, is trained for eight epochs at $10^{-4}$, and has approximately 217,000 generator parameters with one outer denoising step. All model and adapter choices are frozen on validation.

Writing $\mathcal L_{\mathrm{cand}}$ for candidate imitation, $\mathcal L_{\mathrm{clause}}$ for clause consistency, $\mathcal L_{\mathrm{pref}}$ for exact and teacher action preference, and $\mathcal L_{\mathrm{replay}}$ for verified-solution replay, the student minimizes
\begin{equation}
\begin{aligned}
\mathcal L_{\mathrm{SAT\ student}}={}&\mathcal L_{\mathrm{cand}}+0.05\mathcal L_{\mathrm{clause}}\\
&+0.1\mathcal L_{\mathrm{pref}}+0.2\mathcal L_{\mathrm{replay}}.
\end{aligned}
\end{equation}
The clause term averages the probability that every literal of a clause is false under the predicted Bernoulli probabilities, respecting literal polarity; it couples variables through the formula. Unsatisfiable formulas contribute a search preference without a replay target.

\subsection{Interfaces and evaluation}
For the four constraint-programming backends, let $z$ be the student's global candidate and $u_j(s)$ the solver's assigned value for a fixed variable and $z_j$ otherwise. For a selected variable $i$, the adapter compares
\begin{equation}
S_i(b;s,z)=\sum_{C\ni i}U_{C\setminus i}(u(s))\,V_{Ci}(b),
\end{equation}
where $U_{C\setminus i}$ is one when no literal on another variable supports $C$, and $V_{Ci}(b)$ is one when setting $x_i=b$ satisfies a literal of $C$. The adapter tries the value with larger support, breaks ties with $z_i$, and keeps existing assignments authoritative; the alternative value remains in search. Glucose 4.2.1 and MapleChrono initialize their native phase preference from $z$ through PySAT, after which phase saving, learning, and restarts proceed normally. The Glucose build checks interrupts and budgets at a safe loop boundary in both variants; on exhaustion it cancels to the root and returns an unfinished result. Audit cases confirm identical unlimited-search results and counters with the unmodified build.

Twelve methods run serially on fixed CPU resources and one GPU in 12 cyclic orders and their reversals, giving 24 rounds; each round regenerates advice for all 500 formulas under a one-second budget. Per-instance time includes preparation, construction, and solve work plus the instance's allocated share of batched generation; shared generation is charged in full to every guided solver. Loading and warmup precede measurement. In each round, common-cohort statistics use the intersection of native and guided solved sets, independent statistics use each method's own set, and means, medians, and P95 are computed within rounds and averaged over the 24 rounds with equal weight. Choco, LCG, Gecode, and JaCoP report nodes and backtracks; Glucose and MapleChrono report decisions and conflicts. Timeouts stay in the raw results and are excluded from solved-cohort statistics.

\section{Theoretical Details}
\label{app:theory}
\subsection{Action cost and prediction bound}
The state $s=(x,D,h)$ includes every component that affects continuation: adaptive heuristic weights, restart state, and the fixed rule for remaining value order. For each candidate $a$, $Q_{\mathcal S}(s,a)$ clones the \emph{same} state, forces $a$ first in the selected variable, and counts nodes to the first solution or a complete unsatisfiability proof. Training uses root states, so all candidates share the same heuristic history; runs that reach the node limit are censored. The bound below applies to the measured continuation cost of the adaptive solver.

Fix $s$ and write $q(a)=\log(1+Q_{\mathcal S}(s,a))$. If every candidate prediction has error at most $\epsilon$, then
\begin{align}
q(\hat a)&\le r_\theta(s,\hat a)+\epsilon
\le r_\theta(s,a^*)+\epsilon
\le q(a^*)+2\epsilon.
\end{align}
Exponentiation gives Theorem~\ref{thm:cost_prediction}, and the same inequality yields
\begin{equation}
0\le Q_{\mathcal S}(s,\hat a)-Q_{\mathcal S}(s,a^*)
\le (e^{2\epsilon}-1)(1+Q_{\mathcal S}(s,a^*)).
\end{equation}
This comparison holds under the same continuation, including adaptive behavior, with no fixed-tree assumption.

Suppose $a^*$ is unique and $\Delta=\min_{a\ne a^*}(q(a)-q(a^*))$. If $2\epsilon<\Delta$, every nonoptimal candidate has predicted cost strictly above $a^*$, so the learned choice is exactly optimal at this state. A reversal across a large measured cost gap changes the search, whereas errors that preserve the cheapest action leave it unchanged. If a teacher has error $\epsilon_T$ and the student differs from it by at most $\eta$ on every candidate, the student has error at most $\epsilon_T+\eta$; the margin and regret results follow with that sum. The experiments measure how the resulting choices affect solver outcomes.

The intervention reorders a finite set of legal values, so constraints and alternatives remain intact. Under exhaustive-search conditions, complete search stays complete: every branch remains reachable and every answer is checked. Hint adapters return control to native search after advice. With a search budget, an unfinished run is a timeout; completeness holds for the underlying search.

\subsection{Proofs and global accounting}
Fix a satisfiable binary state with exactly one solution-leading child. Let $G$ be the common cost of reaching and searching the good child, and $W$ the additional cost of exhausting the failing child, assuming failed-child exploration leaves the good-child continuation unchanged. With $M=1$ when the policy tries the failing child first, the two costs are $G$ and $G+W$, so
\begin{equation}
C_\pi(s)=G+MW,
\end{equation}
and $\mathbb E[C_\pi]=G+(1-p_s(\pi))W$ (Eq.~\eqref{eq:binary_decomp}). Policies with the same $G,W$ differ by $(p_s(\pi_b)-p_s(\pi_a))W$. When $W$ is random, the exact term is $\mathbb E[MW]$. If $M$ and $W$ are uncorrelated, then $\mathbb E[MW]=(1-p_s)\mathbb E[W]$. Since $G$ and $W>0$ are independent of the ordering policy, minimizing expected cost is exactly maximizing $p_s(\pi)$; with policy-dependent overhead $\tau_\pi$, the objective is $\tau_\pi+G+(1-p_s(\pi))W$. For adaptive Choco, the guarantee is the state-conditioned bound of Theorem~\ref{thm:cost_prediction}, which covers continuations that depend on prior failure.

On a \emph{fixed, nonadaptive} depth-first search tree for a unique-solution instance, the solution path has disjoint failing sibling subtrees. With $G$ the path cost, $W_h$ the fixed cost of the $h$th failing sibling, and $M_h$ indicating that the sibling is tried first, the pathwise cost is
\begin{equation}
C_\pi=G+\sum_{h=1}^{H}M_hW_h,
\end{equation}
and linearity of expectation gives
\begin{equation}
\mathbb E[C_\pi]=G+\sum_{h=1}^{H}\bigl(1-p_h(\pi)\bigr)W_h,
\end{equation}
with no independence among the $M_h$ required. These expressions count nodes; model overhead enters through measured full-set time. For Choco's adaptive dom/wdeg tree, an explored failure changes weights and later states, so the governing object is the state-conditioned counterfactual $Q$ that DiBS learns; its benefit is reported as counterfactual action cost at queried states together with end-to-end measured time.
\end{document}